\pdfoutput=1
\documentclass{article}

\usepackage[margin=1in]{geometry}
\usepackage[utf8]{inputenc} 
\usepackage[T1]{fontenc}    
\usepackage{hyperref}       
\usepackage{url}            
\usepackage{booktabs}       
\usepackage{amsfonts}       
\usepackage{nicefrac}       
\usepackage{microtype}      

\usepackage[noend]{algpseudocode}
\usepackage[utf8]{inputenc}
\usepackage{algorithm}
\usepackage{pgfplots}
\usepackage{amsmath}
\usepackage{multirow}
\usepackage{graphicx}
\usepackage{amssymb}
\usepackage{caption}
\usepackage{subcaption}
\usepackage{amsthm}
\usepackage[export]{adjustbox}
\usepackage{natbib}
\usepackage{booktabs}
\usepackage{bm}
\usepackage{color}
\usepackage[normalem]{ulem}
\usepackage{scalerel,stackengine}
\usepackage{tikz}
\useunder{\uline}{\ul}{}

\usepgfplotslibrary{groupplots}

\theoremstyle{plain}
\newtheorem{theorem}{Theorem}[]
\newtheorem{definition}{Definition}[]

\newtheorem{claim}{Claim}[]
\DeclareMathOperator*{\argmax}{arg\,max}

\newcommand\inner[2]{\langle #1, #2 \rangle}

\newcommand{\norm}[1]{\left|\left|{#1}\right|\right|}
\DeclareMathOperator{\sign}{sgn}

\allowdisplaybreaks

\title{Prior Convictions: Black-Box Adversarial \\ Attacks with Bandits and
Priors }

\author{Andrew Ilyas\footnote{Equal contribution} \\
  MIT\\
  \texttt{ailyas@mit.edu} \\
   \and
   Logan Engstrom\footnotemark[1] \\
  MIT\\
  \texttt{engstrom@mit.edu} \\
  \and
  Aleksander M\k{a}dry \\
  MIT\\
  \texttt{madry@mit.edu}
}

\date{}
\begin{document}
\maketitle

\begin{abstract}
        We study the problem of generating adversarial examples in a black-box setting in which only loss-oracle
    access to a model is available. We introduce a framework
    that conceptually unifies much of the existing work on black-box
    attacks, and we demonstrate that the current state-of-the-art
    methods are
    optimal in a natural sense. Despite this optimality, we show
    how to improve black-box attacks by bringing a new element into the
    problem: gradient priors. We give a
    bandit optimization-based algorithm that allows us to seamlessly integrate any such
    priors, and we explicitly identify and incorporate two examples.
    The resulting methods use two to four times fewer queries and fail two
    to five times less often than the current state-of-the-art.
\footnote{The code for reproducing
    our work is available at \url{https://git.io/blackbox-bandits}.}
\end{abstract}

\section{Introduction}
Recent research has shown that neural networks exhibit significant vulnerability to
adversarial examples, or slightly perturbed inputs designed to
fool the network prediction.  This vulnerability is present in a wide range
of settings, from situations in which inputs are fed directly to
classifiers~\citep{szegedy2013intriguing,carlini2016hidden} to highly
variable real-world environments~\citep{KurakinGB16a,turtle}. Researchers
have developed a host of methods to construct such
attacks~\citep{fgsm,deepfool,carlini2017towards,madry2017towards},  most of
which correspond to first order (i.e., gradient based) methods.  These
attacks turn out to be highly effective: in many cases, only a few gradient
steps suffice to construct an adversarial perturbation.

A significant shortcoming of many of these attacks, however, is that they 
fundamentally rely on the {\em white-box} threat model. That is, they crucially
require direct access to the gradient of the classification loss of the attacked
network.  In many real-world situations, expecting this kind of complete access
is not realistic. In such settings, an attacker can only issue classification
queries to the targeted network, which corresponds to a more restrictive {\em
black box} threat model. 

Recent work \citep{zoo,exploring-the-space,query-efficient} provides a
number of attacks for this threat model.
Chen et. al~\cite{zoo} show how to use a basic primitive of zeroth
order optimization, the finite difference method, to estimate the gradient
from classification queries and then use it (in addition to a number of
optimizations) to mount a gradient based
attack. The method indeed successfully constructs  adversarial perturbations.
It comes, however, at the cost of introducing a significant overhead in terms of the number of queries needed. For instance, attacking an ImageNet~\citep{ilsvrc15}
classifier requires hundreds of thousands of queries. Subsequent work
\citep{query-efficient} improves this dependence
significantly, but still falls short of fully mitigating this issue (see
Section~\ref{subsec:results} for a more detailed analysis).

\subsection{Our contributions}

We revisit zeroth-order optimization in the context
of adversarial example generation, both from an empirical
and theoretical perspective. We propose a new approach for generating
black-box adversarial examples, using bandit optimization in order to exploit
prior information about the gradient, which we show is necessary to break
through the optimality of current methods.  We evaluate our approach on the task of
generating black-box adversarial examples, where the methods obtained from
integrating two example priors significantly outperform state-of-the-art approaches.

Concretely, in this work:
\begin{enumerate}
    \item We formalize the gradient estimation problem as the central problem in the context of query-efficient black-box attacks. We then show how the resulting framework unifies the previous attack methodology. We prove that the least squares method, a classic primitive in signal processing, not only constitutes an optimal solution to the general gradient estimation problem but also is essentially equivalent to the current-best black-box attack methods.
    \item We demonstrate that, despite this seeming optimality of these methods, we can still improve upon them by exploiting an aspect of the problem that has been not considered previously: the priors we have on the distribution of the gradient. We identify two example classes of such priors, and show that they indeed lead to better predictors of the gradient.
    \item Finally, we develop a bandit optimization framework for generating
	black-box adversarial examples which allows for the seamless
	integration of priors. To demonstrate its effectiveness, we show that leveraging the two aforementioned priors
	yields black-box attacks that are 2-5 times more query
	efficient and less failure-prone than the state of the art.
\end{enumerate}

\begin{table}[h]
\centering
\caption{Summary of effectiveness of $\ell_2$ and $\ell_\infty$ ImageNet
attacks on Inception v3 using NES, bandits with time prior (Bandits$_{T}$),
and bandits with time and data-dependent priors (Bandits$_{TD}$). Note that
in the first column, the average number of queries  is calculated only over
successful attacks, and we enforce a query limit of 10,000 queries. For
purposes of direct comparison, the last column calculates the average
number of queries used for only the images that NES (previous SOTA) was
successful on. Our most powerful attack uses 2-4 times fewer queries, and
fails 2-5 times less often.
}
\label{tab:multistep}
  \begin{tabular}{@{}cccccccccc@{}}
    \toprule
    \multirow{2}{*}{\textbf{Attack}} & \phantom{x} &
    \multicolumn{2}{c}{\textbf{Avg. Queries}} & \phantom{x} &
    \multicolumn{2}{c}{\textbf{Failure Rate}} & \phantom{x} &
    \multicolumn{2}{c}{\textbf{Queries on NES Success}} \\ 
    \cmidrule{3-4} \cmidrule{6-7} \cmidrule{9-10} && $\ell_\infty$ &
    $\ell_2$ && $\ell_\infty$ & $\ell_2$ && $\ell_\infty$ & $\ell_2$ \\
    \midrule
    NES && 1735 & 2938 && 22.2\% & 34.4\% && 1735 & 2938\\
    Bandits$_{T}$ && 1781 & 2690 && 11.6\% & 30.4\% && 1214 & 2421 \\
    \textbf{Bandits$_{\mathbf{TD}}$} && \textbf{1117} & \textbf{1858} &&
    \textbf{4.6\%} &  \textbf{15.5\%} && \textbf{703} & \textbf{999} \\
    \bottomrule
  \end{tabular}
\end{table}

\section{Black-box attacks and the gradient estimation problem}
Adversarial examples are natural inputs to a machine learning system that have 
been carefully perturbed in order to induce misbehaviour of the system, under a
constraint on the magnitude of the pertubation (under some metric). For image
classifiers, this misbehaviour can be either classification
as a specific class other than the original one (the targeted attack) or misclassification (the untargeted attack).
For simplicity and to make the presentation of the overarching framework
focused, in this paper we restrict our attention to the untargeted case. Both our algorithms and the whole framework can be, however, easily adapted to the targeted setting.
Also, we consider the most standard threat model in which adversarial perturbations must have $\ell_p$-norm, for some fixed $p$, less than some $\epsilon_p$.

\subsection{First-order adversarial attacks}
Suppose that we have some classifier $C(x)$ with a corresponding classification loss function $L(x,y)$,
where $x$ is some input and $y$ its corresponding label.
In order to generate a misclassified input from some input-label pair
$(x,y)$, we want to find an adversarial example $x'$ which maximizes $L(x',y)$ but still
remains $\epsilon_p$-close to the original input. We can thus
formulate our adversarial attack problem as the following constrained optimization
task:
\[
  x' = \argmax_{x' :  \|x' - x\|_p \leq \epsilon_p} L(x',y)
\]
First order methods tend to be very successful at solving the problem despite its non-convexity~\citep{fgsm,carlini2017towards,madry2017towards}. A first order method
used as the backbone of some of the most powerful white-box adversarial attacks for $\ell_p$
bounded adversaries is {\em projected gradient descent (PGD)}. This iterative
method, given some input $x$ and its correct label $y$, computes a perturbed input $x_k$ by applying $k$ steps of the following update (with $x_0=x$)
\begin{equation}
  \label{eq:pgd_wb} x_{l} = \Pi_{B_p(x, \epsilon)} (x_{l-1} + \eta s_l) 
  \qquad\qquad\qquad\text{ with }
  s_{l} = \Pi_{\partial B_p(0, 1)} \nabla_{x} L(x_{l-1}, y)
\end{equation}

Here, $\Pi_{S}$ is the projection onto the set $S$, $B_p(x', \varepsilon')$ is
the $\ell_p$ ball of radius $\varepsilon'$ around $x'$, $\eta$ is the step size,
and $\partial U$ is the boundary of a set $U$. Also, as is standard in
continuous optimization, we make $s_l$ be the projection of the gradient
$\nabla_x L(x_{l-1}, y)$ at $x_{l-1}$ onto the unit $\ell_p$ ball. This way we
ensure that $s_{l}$ corresponds to the unit $\ell_p$-norm vector that has the
largest inner product with $\nabla_x L(x_{l-1}, y)$. (Note that, in the case of
the $\ell_2$-norm, $s_l$ is simply the normalized gradient but in the case of,
e.g., the $\ell_\infty$-norm, $s_l$ corresponds to the sign vector, $\sign\left(\nabla_{x} L(x_{l-1}, y)\right)$ of the gradient.)

So, intuitively, the PGD update perturbs the input in the direction that (locally) increases the loss the most.
Observe that due to the projection in ~\eqref{eq:pgd_wb}, $x_k$ is
always a valid perturbation of $x$, as desired.

\subsection{Black-box adversarial attacks}
The projected gradient descent (PGD) method described above is designed to be used in the context of so-called {\em white-box} attacks. That is, in the setting where the adversary has full access to the gradient $\nabla_x L(x, y)$ of the loss function of the attacked model. 
In many practical scenarios, however, this kind of access is not
available---in the corresponding, more
realistic {\em black-box} setting, the adversary has only access to an oracle
that returns for a given input $(x,y)$, only the value of the loss $L(x,y)$.

One might expect that PGD is thus not useful in such black-box setting. It turns out, however, that this intuition is incorrect. Specifically, one can still {\em estimate} the gradient using only such value queries. (In fact, this kind of estimator is the backbone of so-called zeroth-order optimization frameworks \citep{spall2005introduction}.) The most canonical primitive in this context is the {\em finite difference method}. This method estimates the {\em directional} derivative $D_{v} f(x) = \inner{\nabla_x f(x)}{v}$ of some function $f$ at a point $x$ in the direction of a vector $v$ as
\begin{equation}
\label{eq:direct_fd} D_v f(x) = \inner{\nabla_x f(x)}{v} \approx \left(f(x+\delta v)-f(x)\right)/\delta.
\end{equation}
Here, the step size $\delta > 0$ governs the quality of the
gradient estimate. Smaller $\delta$ gives more accurate estimates but also decreases reliability, due to precision and noise issues. Consequently, in practice, $\delta$ is a tunable parameter.
Now, we can just use finite differences to construct an estimate of the gradient. To this end, one can find the $d$ components of the gradient by estimating the inner products of the gradient with all the standard basis vectors $e_1, \ldots, e_d$:  
\begin{equation}
  \label{eq:zoo_approx} \widehat{\nabla}_x L(x, y) = \sum_{k=1}^d e_k \left(L(x+\delta e_k, y)-L(x, y)\right)/\delta \approx \sum_{k=1}^d e_k \inner{\nabla_x L(x,y)}{e_k}
\end{equation}
We can then easily implement the PGD attack (c.f. \eqref{eq:pgd_wb}) using this estimator: 
\begin{equation}
  \label{eq:pgd_wb_est} x_{l} = \Pi_{B_p(x, \epsilon)} (x_{l-1} + \eta
  \widehat{s}_l)\qquad\qquad\qquad\text{ with  }\ \
  \widehat{s}_l = \Pi_{\partial B_p(0, 1)}  \widehat{\nabla}_{x} L(x_{l-1}, y)
\end{equation}
Indeed, \cite{zoo} were the first to use finite differences methods in this basic form to power PGD--based adversarial attack in the black-box setting. This basic attack was shown to be successful but, since its query complexity is proportional to the dimension, its resulting query complexity was prohibitively large. For example, the Inception v3~\citep{szegedy2015rethinking} classifier on the ImageNet dataset has dimensionality d=268,203 and thus this method would require 268,204 queries. (It is worth noting, however, that \cite{zoo} developed additional methods to, at least partially, reduce this query complexity.)
\begin{figure}[h]
    \begin{center}
	\includegraphics{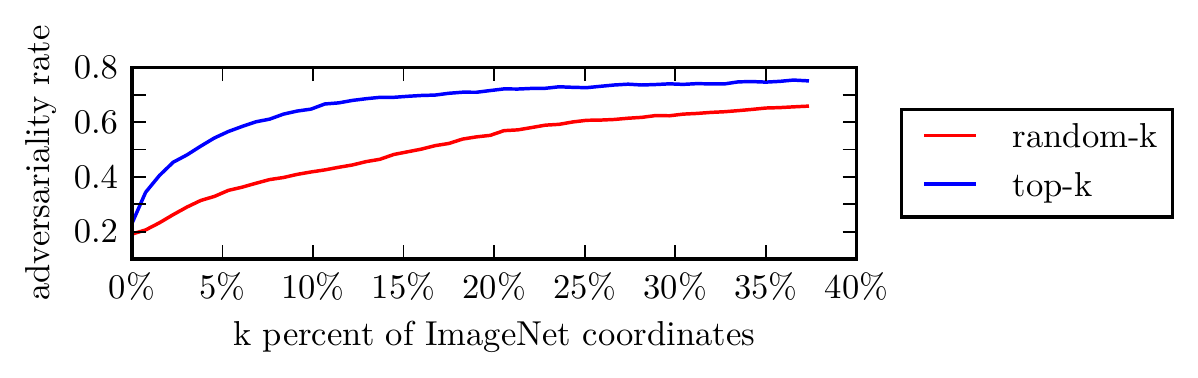}
    \end{center} 
    \caption{The fraction of correctly estimated coordinates of
	$\sign(\nabla_x L(x,y))$ required to successfully execute the
    single-step PGD (also known as FGSM) attack, with $\epsilon=0.05$. In
the experiment, for each $k$, the top $k$ percent -- chosen either by
magnitude (\texttt{top-k}) or randomly (\texttt{random-k}) -- of the signs
of the coordinates are set correctly, and the rest are set to $+1$ or $-1$
at random. The adversariality rate is the portion of 1,000 random ImageNet
images misclassified after one  FGSM step. Observe that, for example,
estimating only 20\% of the coordinates correctly leads to
misclassification in the case of more than 60\% of images.}
\label{fig:meankgraph}
\end{figure}

\subsection{Black-box attacks with imperfect gradient estimators} 
In the light of the above discussion, one can wonder if the algorithm \eqref{eq:pgd_wb_est} can be made more query-efficient. A natural idea here would be to avoid fully estimating the gradient and rely instead only on its {\em imperfect} estimators. This gives rise to the following question: \emph{How accurate of an gradient estimate is necessary to execute a successful PGD attack?}

We examine this question first in the simplest possible setting: one in which we only take a {\em single} PGD step (i.e., the case of $k=1$). Previous work ~\citep{fgsm} indicates that such an attack can already be quite powerful. So, we study how the effectiveness of this attack varies with gradient estimator accuracy.
Our experiments, shown in Figure~\ref{fig:meankgraph}, suggest that it is feasible to generate
adversarial examples without estimating correctly even most of the coordinates of the
gradient. For example, in the context of $\ell_\infty$ attacks, setting a randomly selected 20\% of the coordinates in the gradient
to match the true gradient (and making the remaining coordinates have random sign) is sufficient to fool the classifier
on more than 60\% images with single-step PGD.
Our experiments thus demonstrate that an adversary is likely to be able to cause a misclassification by performing the iterated PGD attack, even when driven by a gradient estimate that is largely imperfect.

\subsection{The gradient estimation problem}
\label{sec:gradient_estimation_problem}
The above discussion makes it clear that successful attacks do not require
a perfect gradient estimation, provided this estimate is suitably
constructed.  It is still unclear, however, how to efficiently find this
kind of imperfect but helpful estimator. Continuous optimization
methodology suggests that the key characteristic needed from our
estimator is for it to have a sufficiently large inner product with the
actual gradient. We thus capture this challenge as the following {\em
gradient estimation problem}:
\begin{definition}[Gradient estimation problem]\label{def:gradient_estimation}
For an input/label pair $(x, y)$ and a loss function $L$, let $g^* =
\nabla_x L(x, y)$ be the gradient of $L$ at $(x,y)$. Then the goal of the
gradient estimation problem is to find a unit vector $\widehat{g}$ maximizing the inner product 
    \begin{equation}
    \label{eq:gep}
          \mathbb{E}\left[ \widehat{g}^Tg^*\right],
  \end{equation}
  from a limited number of (possibly adaptive) function value queries $L(x', y')$. (The expectation here is taken over the randomness of the estimation algorithm.)  
\end{definition}

One useful perspective on the above gradient estimation problem stems from
casting the recovery of $g^*$ in \eqref{eq:gep} as an underdetermined
vector estimation task. That is, one can view each execution of the finite
difference method (see \eqref{eq:direct_fd}) as computing an inner product
query in which we obtain the value of the inner product of $g^*$ and some
chosen direction vector $A_i$. Now, if we execute $k$ such queries, and
$k<d$ (which is the regime we are interested in), the information acquired
in this process can be expressed as the following (underdetermined) linear
regression problem $Ag^* = y$, where the rows of the matrix $A$ correspond
to the queries $A_1, \ldots, A_k$ and the entries of the vector $y$ gives
us the corresponding inner product values.

\paragraph{Relation to compressive sensing} The view of the gradient
estimation problem we developed bears striking similarity to the
compressive sensing setting~\citep{foucart2013mathematical}. Thus one might
wonder if the toolkit of that area could be applied here. Compressive
sensing crucially requires, however, certain sparsity structure in the
estimated signal (here, in the gradient $g^*$) and, to our knowledge, the
loss gradients do not exhibit such a structure. (We discuss this
further in Appendix~\ref{app:omitted}.)

\paragraph{The least squares method} In light of this, we turn our attention to
another classical signal-processing method: norm-minimizing $\ell_2$ least
squares estimation. This method approaches the estimation problem posed in
\eqref{eq:gep} by casting it as an undetermined linear regression problem of the
form $Ag^*=b$, where we can choose the matrix $A$ (the rows of $A$ correspond to
inner product queries with $g^{*}$). Then, it obtains the
solution $\widehat{g}$ to the regression problem by solving:

\begin{equation}
\label{eq:lstsq} \min_{\widehat{g}} \|\widehat{g}\|_2 \qquad \text{ s.t. } A\widehat{g} = y.
\end{equation}

A reasonable choice for $A$ (via~\cite{johnson1984extensions} and related results) is the
distance-preserving random Gaussian projection matrix, i.e. $A_{ij}$ normally
distributed.

The resulting algorithm turns out to yield solutions that are approximately
those given by Natural Evolution Strategies (NES), which \citep{query-efficient}
previously applied to black-box attacks. In particular, in Appendix
\ref{app:proofs}, we prove the following theorem.

\begin{theorem}[NES and Least Squares equivalence]
    \label{thm:neslsq}
     Let $\hat{x}_{NES}$ be the Gaussian $k$-query NES estimator of a $d$-dimensional gradient
    $\bm{g}$ and let $\hat{x}_{LSQ}$ be the minimal-norm $k$-query least-squares estimator of $\bm{g}$. For any $p>0$, with probability at least $1-p$ we have that 
    \[
	\inner{\hat{x}_{LSQ}}{\bm{g}} -
	\inner{\hat{x}_{NES}}{\bm{g}} \leq
	O\left(\sqrt{\frac{k}{d}\cdot\log^3\left(\frac{k}{p}\right)}\right)\norm{g}^2.
    \]
\end{theorem}

Note that when we work in the underdetermined setting, i.e., when $k\ll d$
(which is the setting we are interested in), the right hand side bound
becomes vanishingly small. Thus, the equivalence indeed holds. In fact,
using the precise statement (given and proved in
Appendix~\ref{app:proofs}), we can show
that Theorem~\ref{thm:neslsq} provides us with a non-vacuous equivalence
bound. 
 Further, it turns out that one can exploit this equivalence to prove that the algorithm proposed in
\cite{query-efficient} is not only natural but optimal, as the least-squares 
estimate is an information-theoretically optimal gradient estimate in the
regime where $k = d$, and an error-minimizing estimator in the regime where
$k << d$.

\begin{theorem}[Least-squares optimality (Proof in Appendix~\ref{app:proofs})]
    For a linear regression problem $y = A\bm{g}$ with known $A$ and $y$, unknown $\bm{g}$, and
    isotropic Gaussian errors, the least-squares estimator is finite-sample efficient, i.e. the
    minimum-variance unbiased (MVU) estimator of the latent vector $\bm{g}$.
\end{theorem}

\begin{theorem}[Least-squares optimality (Proof in~\cite{biaslinear})]
    In the underdetermined setting, i.e. when $k << d$, the minimum-norm
    least squares estimate ($\hat{x}_{LSQ}$ in Theorem~\ref{thm:neslsq}) is
    the minimum-variance (and thus minimum-error, since bias is fixed)
    estimator with no empirical loss.
\end{theorem}

\section{Black-box adversarial attacks with priors}
\label{sec:approach}
The optimality of least squares strongly suggests that we have reached the
limit of query-efficiency of black-box adversarial attacks. But is this really the case?
Surprisingly, we show that an improvement {\em is} still
possible. The key observation is that the optimality we established of
least-squares (and by Theorem~\ref{thm:neslsq}, the NES approach
in~\citep{query-efficient}) holds only for the most
basic setting of the gradient estimation problem, a setting where we
assume that the target gradient is a truly arbitrary and completely unknown
vector. 

However, in the context we care about this assumption does not hold --
there is actually plenty of prior knowledge about the gradient available.
Firstly, the input with respect to which we compute the gradient is not
arbitrary and exhibits locally predictable structure which is consequently
reflected in the gradient. Secondly, when performing iterative gradient
attacks (e.g. PGD), the gradients used in successive iterations are likely
to be heavily correlated.

The above observations motivate our focus on \textit{prior information} as
an integral element of the gradient estimation problem.
Specifically, we enhance Definition~\ref{def:gradient_estimation} by
making its objective
\begin{equation}
  \label{eq:gep-priors}
  \mathbb{E}\left[ \widehat{g}^Tg^*\right | I], \text{
where $I$ is prior information available to us.}
\end{equation}

This change in perspective gives rise to two important questions:
\textit{does there exist prior information that can be useful to us?}, and
\textit{does there exist an algorithmic way to exploit this information?}
We show that the answer to both of these questions is affirmative.

\subsection{Gradient priors}
Consider a gradient $\nabla_x L(x, y)$ of the loss function corresponding
to some input $(x,y)$. Does there exist some kind of prior that can be
extracted from the dataset $\{x_i\}$, in general, and the input $(x,y)$ in
particular, that can be used as a predictor of the gradient? We demonstrate that
it is indeed the case, and give two example classes of such priors.

\paragraph{Time-dependent priors} The first class of priors we consider are
time-dependent priors, a standard example of which is what we refer to as
the ``multi-step prior.'' We
find that along the trajectory taken by estimated gradients, successive
gradients are in fact heavily correlated. We show
this empirically by taking steps along the optimization path generated by
running the NES estimator at each point, and plotting the normalized inner
product (cosine similarity) between successive gradients, given by
\begin{equation}
    \label{eq:cosine_similarity}
    \frac{\langle \nabla_x L(x_t, y), \nabla_x L(x_{t+1}, y) \rangle}{||\nabla_x
    L(x_t, y)||_2  ||\nabla_x L(x_{t+1}, y)||_2} \qquad t \in \{1\ldots T-1\}.
\end{equation}

\begin{figure}[htbp]
  \begin{minipage}[t]{0.5\linewidth-0.34cm}
    \begin{centering}
      \includegraphics{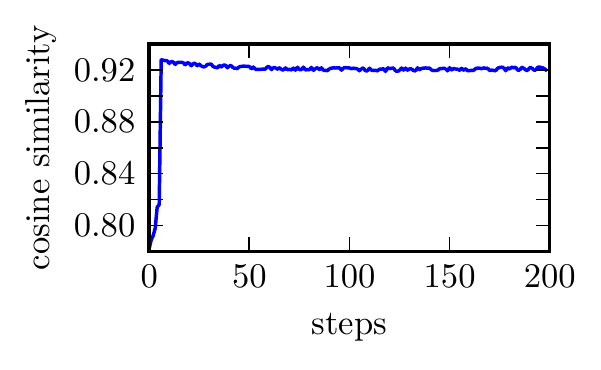}
    \end{centering}
    \caption{Cosine similarity between the gradients at the current
	and previous steps along the optimization trajectory of NES PGD
attacks, averaged over 1000 random ImageNet images.}
    \label{fig:successive_correlation}
  \end{minipage}
  \hspace{0.5cm}
  \begin{minipage}[t]{0.5\linewidth-0.34cm}
    \begin{centering}
      \includegraphics{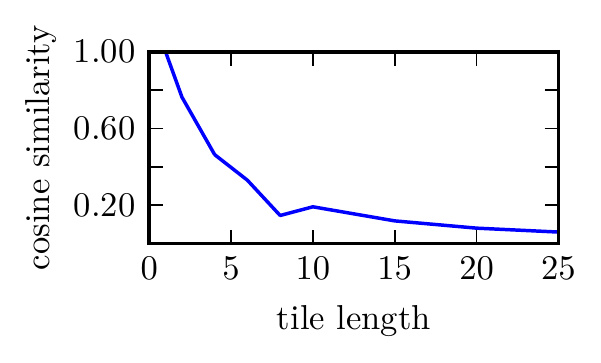}
    \end{centering}
    \caption{Cosine similarity of ``tiled'' image gradient with
	original image gradient versus the length of the square tiles,
averaged over 5,000 randomly selected ImageNet images.}
    \label{fig:tiling}
  \end{minipage}
\end{figure}

Figure~\ref{fig:successive_correlation} demonstrates that there indeed is a non-trivial
correlation between successive gradients---typically, the gradients of
successive steps (using step size from~\cite{query-efficient}) have a cosine
similarity of about 0.9. Successive gradients continue to correlate at
higher step sizes: Appendix~\ref{app:omitted} shows that the trend continues even
at step size 4.0 (a typical value for the \textit{total} perturbation
bound $\varepsilon$). This indicates that there indeed is a potential
gain from incorporating this correlation into our iterative optimization.
To utilize this gain, we intend to use the
gradients at time $t-1$ as a prior for the gradient at time $t$, where both
the prior and the gradient estimate itself evolve over iterations.

\paragraph{Data-dependent priors}
We find that the time-dependent prior discussed above is not the only type of prior
one can exploit here. Namely, we can also use the structure of the inputs
themselves to reduce query complexity (in fact, the existence of such
data-dependent priors is what makes machine learning successful in the
first place).

In the case of image classification, a simple and heavily exploited example
of such a prior stems from the fact that images tend to exhibit a spatially
local similarity (i.e. pixels that are close together tend to be similar).
We find that this similarity also extends to the gradients:
specifically, whenever two coordinates $(i, j)$ and $(k, l)$ of
$\nabla_{x} L(x, y)$ are close, we
expect $\nabla_{x} L(x, y)_{ij} \approx \nabla_{x} L(x,y)_{kl}$
too. To corroborate and quantify this phenomenon, we compare $\nabla_x
L(x,y)$ with an average-pooled, or
``tiled'', version (with ``tile length'' $k$) of the same
signal. An example of such an average-blurred gradient can be seen in
Appendix~\ref{app:omitted}. More concretely, we apply to the gradient the mean
pooling operation with kernel size $(k,k,1)$ and stride $(k,k,1)$, then upscale the spatial dimensions by $k$. We then
measure the cosine similarity between the average-blurred gradient and the
gradient itself. Our results, shown in Figure~\ref{fig:tiling}, demonstrate
that the gradients of images are locally similar enough to allow for
average-blurred gradients to maintain relatively high cosine similarity
with the actual gradients, even when the tiles are large. Our results suggest
that we can reduce the dimensionality of our problem by a factor of $k^2$
(for reasonably large $k$) and still estimate a vector pointing close to
the same direction as the original gradient. This factor, as we show later,
leads to significantly improved black-box adversarial attack performance.

\subsection{A framework for gradient estimation with priors}
Given the availability of these informative gradient priors, we now
need a framework that enables us to easily incorporate these priors into
our construction of black-box adversarial attacks. Our proposed method builds on the
framework of \textit{bandit optimization}, a fundamental tool in online
convex optimization~\cite{hazan-oco}. In the bandit optimization framework,
an agent plays a game that consists of a sequence of rounds. In round $t$,
the agent must choose a valid action, and then by playing the action incurs a loss given by a
loss function $\ell_t(\cdot)$ that is unknown to the agent. After playing
the action, he/she only learns the loss
that the chosen action incurs; the loss function is specific to
the round $t$ and may change arbitrarily between rounds. The goal of the
agent is to minimize the average loss incurred over all rounds, and the success
of the agent is usually quantified by comparing the total loss incurred to
that of the \textit{best expert in hindsight} (the best single-action
policy). By the nature of this formulation, the rounds of this game can not
be treated as independent --- to perform well, the agent needs to keep
track of some latent record that aggregates information learned over a
sequence of rounds. This latent record usually takes a form of a
vector $v_t$ that is constrained to a specified (convex) set
$\mathcal{K}$. As we will see, this aspect of the bandit
optimization framework will provide us with a convenient way to incorporate
prior information into our gradient prediction.

\paragraph{An overview of gradient estimation with bandits.} We can cast
the gradient estimation problem as an bandit optimization problem
in a fairly direct manner. Specifically, we let the action at each round
$t$ be a gradient estimate $g_t$ (based on our latent vector $v_t$), and the loss $\ell_t$
correspond to the (negative) inner product between this prediction and the
actual gradient. Note that we will never have a direct access to this loss
function $\ell_t$ but we are able to evaluate its value on a particular
prediction vector $g_t$ via the finite differences method
\eqref{eq:direct_fd} (which is all that the bandits optimization
framework requires us to be able to do).

Just as this choice of the loss function $\ell_t$ allows us to quantify
performance on the gradient estimation problem, the 
latent vector $v_t$ will allow us to algorithmically incorporate prior
information into our predictions.
Looking at the two example priors we consider, the time-dependent prior will be
reflected by carrying over the latent vector between the gradient
estimations at different points. Data-dependent priors will be captured
by enforcing that our latent vector has a particular structure. For the
specific prior we quantify in the preceding section (data-dependent prior
for images), we will simply reduce the dimensionality of the latent vector via
average-pooling (``tiling``), removing the need for extra queries
to discern components of the gradient that are spatially close.

\subsection{Implementing gradient estimation in the bandit framework}
We now describe our bandit framework for adversarial example generation in
more detail. Note that the algorithm is general and can be used to
construct black-box adversarial examples where the perturbation is
constrained to any convex set ($\ell_p$-norm
constraints being a special case). We discuss the algorithm in
its general form, and then provide versions explicitly applied to
the $\ell_2$ and $\ell_\infty$ cases.

As previously mentioned, the latent vector $v_t \in \mathcal{K}$ serves as
a prior on the gradient for the corresponding round $t$ -- in fact, we make
our prediction $g_t$ be exactly $v_t$ projected onto the appropriate space,
and thus we set $\mathcal{K}$ to be an extension of the space of valid
adversarial perturbations (e.g. $\mathbb{R}^n$ for $\ell_2$ examples,
$[-1,1]^n$ for $\ell_\infty$ examples). Our loss function $\ell_t$
is defined as
\begin{equation}
    \ell_t(g) = -\inner{\nabla L(x, y)}{\frac{g}{\norm{g}}},
\end{equation}
for a given gradient estimate $g$, where we access this inner product via
finite differences. Here, $L(x, y)$ is the classification loss on an image $x$ with true class
$y$. 

The crucial element of our algorithm will thus be the method of updating the latent vector $v_t$. 
We will adapt here the canonical
``reduction from bandit information''~\citep{hazan-oco}.
Specifically, our update procedure is parametrized by an estimator
$\Delta_t$ of the gradient $\nabla_v \ell_t(v)$, and a first-order update
step $\mathcal{A}\ (\mathcal{K} \times \mathbb{R}^{\dim(\mathcal{K})} \rightarrow
\mathcal{K}),$ 
which maps the latent vector $v_t$ and the estimated gradient of $\ell_t$
with respect to $v_t$ (which we denote $\Delta_t$) to a new latent vector $v_{t+1}$. The
resulting general algorithm is presented as Algorithm \ref{alg:general}.

\begin{algorithm}
\caption{Gradient Estimation with Bandit Optimization}\label{alg:general}
\begin{algorithmic}[1]
    \Procedure{Bandit-Opt-Loss-Grad-Est}{$x, y_{init}$} 
    \State $v_0 \gets \mathcal{A}(\phi) $ 
    \For{each round $t=1, \ldots, T$}
    \State $//$ Our loss in round $t$ is $\ell_t(g_t) = - \inner{\nabla_x
    L(x, y_{init})}{g_t}$
    \State $g_t \gets v_{t-1}$
    \State $\Delta_t \gets \textsc{Grad-Est}(x, y_{init}, v_{t-1})\ //$
    Estimated Gradient of $\ell_t$
    \State $v_{t} \gets \mathcal{A}(v_{t-1}, \Delta_t)$
    \EndFor
    \State $g\gets v_T$
    \State\Return $\Pi_{\partial \mathcal{K}}\left[g\right]$
\EndProcedure
\end{algorithmic}
\end{algorithm}

In our setting, we make the estimator $\Delta$ of the gradient $-\nabla_v
\inner{\nabla L(x, y)}{v}$ of the loss $\ell$ be the standard spherical
gradient estimator (see \cite{hazan-oco}). We take a two-query estimate of
the expectation, and employ antithetic sampling which results in the
estimate being computed as

\begin{equation}
    \Delta = \frac{\ell(v+\delta\bm{u}) -
\ell(v-\delta\bm{u})}{\delta}\bm{u},
\end{equation}
where $\bm{u}$ is a Gaussian vector sampled from $\mathcal{N}(0, \frac{1}{d}I)$.
The resulting algorithm for calculating the gradient estimate given the current latent vector $v$, input $x$ and the initial label $y$ is Algorithm~\ref{alg:grad-est}. 

\begin{algorithm}
\caption{Single-query spherical estimate of $\nabla_v \inner{\nabla L(x,
y)}{v}$}\label{alg:grad-est}
\begin{algorithmic}[1]
    \Procedure{Grad-Est}{$x, y, v$} 
    \State $u \gets \mathcal{N}(0, \frac{1}{d}I)\ //$ Query vector 
    \State $\{q_1, q_2\} \gets \{v + \delta\bm{u}, v - \delta\bm{u}\}\
    //$ Antithetic samples
    \State $\ell_t(q_1) = - \inner{\nabla L(x, y)}{q_1} \approx
    \frac{L(x,y) - L(x+\epsilon\cdot q_1, y)}{\epsilon}\ //$ Gradient
    estimation loss at $q_1$ 
    \State $\ell_t(q_2) = - \inner{\nabla L(x, y)}{q_2} \approx
    \frac{L(x,y)-L(x+\epsilon\cdot q_2, y)}{\epsilon} //$ Gradient
    estimation loss at $q_2$ 
\State $\bm{\Delta} \gets \frac{\ell_t(q_1) - \ell_t(q_2)}{\delta}\bm{u} =  \frac{L(x+\epsilon q_2, y) -
    L(x+\epsilon q_1, y)}{\delta\epsilon}\bm{u}$
    \State $//$ Note that due to cancellations we can actually evaluate
    $\bm{\Delta}$ with only two queries to $L$
    \State \Return $\bm{\Delta}$
\EndProcedure
\end{algorithmic}
\end{algorithm}

A crucial point here is that the above gradient estimator $\Delta_t$ 
parameterizing the bandit reduction has no direct relation to the ``gradient
estimation problem'' as defined in Section
\ref{sec:gradient_estimation_problem}. It is simply a general mechanism by which we
can update the latent vector $v_t$ in bandit optimization. It is the
actions $g_t$ (equal
to $v_t$) which provide proposed solutions to the gradient estimation
problem from Section \ref{sec:gradient_estimation_problem}.

The choice of the update rule $\mathcal{A}$ tends to be natural once the
convex set $\mathcal{K}$ is known. For $\mathcal{K} = \mathbb{R}^n$, we can simply use
gradient ascent:
\begin{equation}
    v_{t} = \mathcal{A}(v_{t-1}, \Delta_{t}) := v_{t-1} + \eta\cdot\bm{\Delta}_{t}
\end{equation}

and the exponentiated gradients (EG) update when the constraint is an
$\ell_\infty$ bound (i.e. $\mathcal{K} = [-1, 1]^n$):
\begin{align*}
    p_{t-1} &= \frac{1}{2}\left(v_{t-1}+1\right) \\
    p_{t} &= \mathcal{A}(g_{t-1}, \Delta_{t}) :=
    \frac{1}{Z}p_{t-1}\exp(\eta\cdot \Delta_{t}) 
    &\text{ s.t. } Z =
    p_{t-1}\exp(\eta\cdot \Delta_{t}) + (1-p_{t-1})\exp(-\eta\cdot
    \Delta_{t}) \\
    v_t &= 2p_t - 1
\end{align*}

Finally, in order to translate our gradient estimation algorithm
into an efficient method for constructing black-box adversarial examples,
we interleave our iterative gradient estimation algorithm with an iterative
update of the image itself, using the boundary projection of $g_t$ in place
of the gradient (c.f. \eqref{eq:pgd_wb}). This
results in a general, efficient, prior-exploiting algorithm for
constructing black-box adversarial examples. The resulting algorithm in the
$\ell_2$-constrained case is shown in Algorithm~\ref{alg:l2}.

\begin{algorithm}
\caption{Adversarial Example Generation with Bandit Optimization
for $\ell_2$ norm perturbations}\label{alg:l2}
\begin{algorithmic}[1]
    \Procedure{Adversarial-Bandit-L2}{$x_{init}, y_{init}$} 
    \State $//\ C(\cdot)$ returns top class
    \State $v_0 \gets {\bf 0}_{1\times d}\ //$ If data prior, $d < \text{dim}(x)$; $v_t$ ($\Delta_t$) up (down)-sampled before (after) line $8$
    \State $x_0 \gets x_{init}\ //$ Adversarial image to be constructed
    \While{$C(x) = y_{init}$}
    \State $g_t \gets v_{t-1}$
    \State $x_{t} \gets x_{t-1} + h \cdot \frac{g_t}{||g_t||_2}\ //$
    Boundary projection $\frac{g}{||g_t||}$ standard PGD: c.f.~\citep{mlnotes} 
    \State $\Delta_t \gets \textsc{Grad-Est}(x_{t-1}, y_{init}, v_{t-1})\ //$
    Estimated Gradient of $\ell_t$
    \State $v_t \gets v_{t-1} + \eta \cdot \Delta_t$
    \State $t\gets t+1$
    \EndWhile
    \Return $x_{t-1}$
\EndProcedure
\end{algorithmic}
\end{algorithm}

\section{Experiments and evaluation}
We evaluate our bandit approach described in Section~\ref{sec:approach} and
the natural evolutionary strategies (NES) approach of
\cite{query-efficient} on their effectiveness in generating untargeted
adversarial examples. We consider both the $\ell_2$ and $\ell_\infty$ threat models
on the ImageNet~\citep{ilsvrc15} dataset, in terms of success rate and query
complexity. We give results for attacks on the Inception-v3, Resnet-50, and
VGG16 classifiers. We further investigate loss
and gradient estimate quality over the optimization trajectory in each
method. 

In evaluating our approach, we test both the bandit approach with time
prior (Bandits$_T$), and our bandit approach with
the given examples of both the data and time priors (Bandits$_{TD}$).
We use 10,000 randomly selected images (scaled to $[0,1]$) to
evaluate all approaches. For NES, Bandits$_T$, and
Bandits$_{TD}$ we found hyperparameters (given in
Appendix~\ref{app:hyperparameters}, along with the experimental parameters) via grid search. 

\subsection{Results}
\label{subsec:results}
For ImageNet, we record the effectiveness of the different approaches in both threat
models in Table~\ref{tab:multistep} ($\ell_2$ and $\ell_\infty$ perturbation
constraints), where we show the attack success rate
and the mean number of queries (of the successful attacks) needed to generate an
adversarial example for the Inception-v3 classifier (results for other
classifiers in Appendix~\ref{app:classifiers}). For all attacks, we limit
the attacker to at most 10,000 oracle queries. As shown in
Table~\ref{tab:multistep}, our bandits framework with both
data-dependent and time prior (Bandits$_{TD}$),  is six and three times
less failure-prone than the previous state of the
art (NES~\citep{query-efficient}) in the $\ell_\infty$
and $\ell_2$ settings, respectively. Despite the higher success rate, our
method actually uses around half as many queries as NES. In particular,
when restricted to the inputs on which NES is successful in generating
adversarial examples,
our attacks are 2.5 and 5 times as query-efficient for the $\ell_\infty$
and $\ell_2$ settings, respectively. 

We also further quantify the
performance of our methods in terms of black-box attacks, and gradient
estimation. Specifically, we first measure average queries
per success after reaching a certain success rate
(Figure~\ref{fig:amalgam}a), which indicates the
dependence of the query count on the desired success rate. The data shows
that for any fixed success rate, our methods are more query-efficient than
NES, and (due to the exponential trend) suggest
that the difference may be amplified for higher success rates. We then
plot the loss of the classifier over time (averaged over all
images), and performance on the gradient
estimation problem for both $\ell_\infty$ and $\ell_2$ cases (which,
crucially, corresponds directly to the expectation we maximize
in~\eqref{eq:gep-priors}. We show these three plots
for $\bm{\ell_\infty}$ in Figure~\ref{fig:amalgam}, and show the
results for $\ell_2$ (which are extremely similar) in
Appendix~\ref{app:results}, along with CDFs showing the success of each
method as a function of the query limit. We find that on every metric in
both threat models, our methods strictly dominate NES in terms of
performance.

\begin{figure}[h]
  \begin{tikzpicture}[every plot/.append style={very thick}]

\begin{groupplot}[group style={group size=3 by 1, horizontal sep=2cm}]
\nextgroupplot[
title={Avg. queries by success \% },
ylabel={\# queries},
x label style={at={(axis description cs:0.5,0.05)},anchor=north},
y label style={at={(axis description cs:0.05,0.5)},anchor=north},
ytick={0,500,1000,1500,2000},
yticklabels={0,,,,2k},
scaled y ticks=false,
axis line style=thick,
xlabel={success rate},
ymin=0, ymax=2000,
xmin=0, xmax=1.0,
axis on top,
xtick={0,0.5,1.0},
xticklabels={0,50,100},
tick pos=both,
width=0.3\textwidth,
height=3cm,
legend cell align={left},
legend columns=-1,
legend entries={NES$\qquad$,Bandits$_T$ (time prior)$\qquad$,Bandits$_{TD}$ (time + data)},
legend style={at={(0.2,1.6)}, anchor={south west}},
]
\addplot [dotted, blue] table[x index=1, y index=0]{fig/data/avgq-vs-success/nes-linf-cdf.csv};
\addplot [dashed, green!50.0!black] table[x index=1, y index=0] {fig/data/avgq-vs-success/banditst-linf-cdf.csv};
\addplot [red] table[x index=1, y index=0] {fig/data/avgq-vs-success/banditstd-linf-cdf.csv};
\nextgroupplot[
title={Average loss},
xlabel={iteration},
ylabel={loss},
x label style={at={(axis description cs:0.5,0.05)},anchor=north},
y label style={at={(axis description cs:0.05,0.5)},anchor=north},
axis line style=thick,
xtick={0,2500,5000,7500,10000},
xticklabels={0,,5k,,10k},
scaled x ticks=false,
xmin=0, xmax=10000,
ymin=0, ymax=18,
axis on top,
ytick={0,3,6,9,12,15,18},
yticklabels={$0$,,,9,,,$18$},
tick pos=both,
width=0.3\textwidth,
height=3cm
]
\addplot [dotted, blue, forget plot] table {fig/data/avg-loss/nes-linf-cdf.csv};
\addplot [dashed, green!50.0!black, forget plot] table {fig/data/avg-loss/bandits-t-linf-cdf.csv};
\addplot [red, forget plot] table {fig/data/avg-loss/bandits-td-linf-cdf.csv};
\nextgroupplot[
title={Correlation with $g^*$ },
xlabel={iteration},
ylabel={cos similarity},
xmin=0, xmax=10000,
ymin=0, ymax=0.1,
axis on top,
x label style={at={(axis description cs:0.5,0.05)},anchor=north},
y label style={at={(axis description cs:0.05,0.5)},anchor=north},
xtick={0,2500,5000,7500,10000},
xticklabels={0,,5k,,10k},
scaled x ticks=false,
scaled y ticks=false,
ytick={0,0.05,0.1},
yticklabels={$0.0$,,$0.1$},
tick pos=both,
axis line style=thick,
width=0.3\textwidth,
height=3cm,
]
\addplot [dotted,blue] table {fig/data/cosine-dists/nes-linf.csv};
\addplot [dashed,green!50.0!black] table {fig/data/cosine-dists/bandits-t-linf.csv};
\addplot [red] table {fig/data/cosine-dists/bandits-td-linf.csv};
\end{groupplot}

\end{tikzpicture}
\vspace*{-1em}
  \caption{\textbf{(left)} Average number of queries per successful image as a
  function of the number of total successful images; at any
  desired success rate, our methods use significantly less
  queries per successful image than NES, and the trend suggests that this
  gap increases with the desired success rate. \textbf{(center)} The loss over time,
  averaged over all images; 
  \textbf{(right)} The correlation of the latent vector with the
  true gradient $g$, which is precisely the
  gradient estimation objective we define.}
  \label{fig:amalgam}
\end{figure}
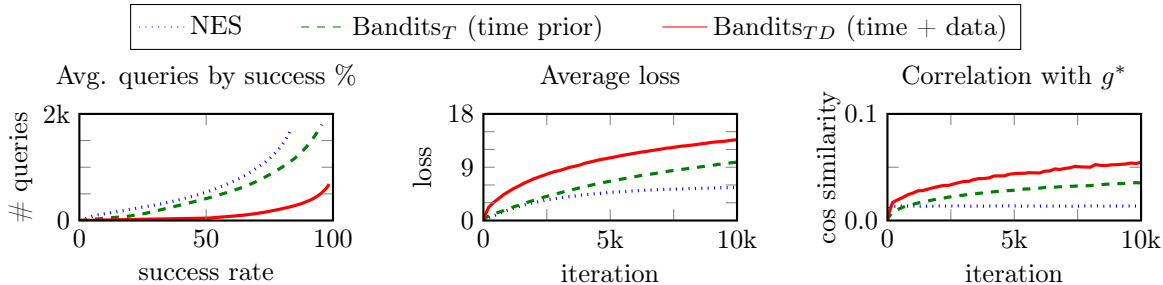

\section{Related work}
All known techniques for generating adversarial examples in the black-box
setting so far rely on either iterative optimization schemes
(our focus) or so-called substitute networks and transferability.

In the first line of work, algorithms use queries to gradually perturb a
given input to maximize a corresponding loss, causing
misclassification. Nelson et. al~\cite{two-class} presented the first such
iterative attack on a special class of binary classifiers.
Later, Xu et. al~\cite{genetic-algos} gave an algorithm for fooling
a real-world system with black-box attacks. Specifically, they fool PDF
document malware classifier by using a genetic algorithms-based attack.
Soon after, Narodytska et. al~\cite{simple-black-box}
described the first black-box attack on {deep neural networks}; the
algorithm uses a greedy search algorithm that selectively changes
individual pixel values. Chen et. al~\cite{zoo} were the first to design
black-box attack based on finite-differences and gradient based
optimization. The method uses coordinate descent to attack black-box neural
networks, and introduces various optimizations to decrease sample
complexity. Building on the work of \cite{zoo}, Ilyas et. al~\cite{query-efficient} designed a black-box attack strategy that also uses
finite differences but via natural evolution strategies (NES) to
estimate the gradients. They then used their algorithm as a primitive in
attacks on more restricted threat models.

In a concurrent line of work, Papernot et. al~\cite{practical-black-box} introduced a method
for attacking models with so-called substitute networks. Here, the attacker
first trains a model -- called a substitute network -- to mimic the target
network's decision boundaries. The attacker then generates adversarial
examples on the substitute network, and uses them to attack the original
target mode. Increasing the rate at which
adversarial examples generated from substitute networks fool the target
model is a key aim of substitute networks work. In \cite{practical-black-box}, the attacker generates a synthetic
dataset of examples labeled by the target classifier using black-box
queries. The attacker then trains a substitute network on the dataset.
Adversarial examples generated with methods developed with recent
methods~\cite{practical-black-box,delving-into}
tend to transfer to a target MNIST classifier.
We note, however, that the overall query efficiency of this type of methods
tends to be worse than that of the gradient estimation based ones. (Their
performance becomes more favorable as one becomes interested in attacking more
and more inputs, as the substitute network has to be trained only once.)

\section{Conclusion}
We develop a new, unifying perspective on black-box adversarial attacks.
This perspective casts the construction of such attacks as a gradient
estimation problem. We prove that a standard least-squares estimator both
captures the existing state-of-the-art approaches to black-box adversarial
attacks, and actually is, in a certain natural sense, an optimal solution
to the problem.

We then break the barrier posed by
this optimality by considering a previously unexplored aspect of the
problem: the fact that there exists plenty of extra prior information about
the gradient that one can exploit to mount a successful adversarial attack.
We identify two examples of such priors: a ``time-dependent'' prior that
corresponds to similarity of the gradients evaluated at similar inputs, and
a ``data-dependent'' prior derived from the latent structure present in the
input space.

Finally, we develop a bandit
optimization approach to black-box adversarial attacks that allows for a seamless integration of such priors. The resulting framework significantly outperforms the state-of-the-art methods, achieving
a factor of two to six improvement in terms of success rate and query
efficiency. Our results thus open a new avenue towards finding priors for
construction of even more efficient black-box adversarial attacks.

\section*{Acknowledgments}
We thank Ludwig Schmidt for suggesting the connection between the least squares method and the natural estimation strategies. 

AI was supported by an Analog Devices Graduate Fellowship.
LE was supported in part by an MIT-IBM Watson AI Lab research grant, the Siebel Scholars Foundation, and NSF Frontier grant CNS-10413920.
AM was supported in part by a Google Research Award, and the NSF grants CCF-1553428 and CNS-1815221.

\newpage
\bibliographystyle{alpha}
\bibliography{paper}

\newpage
\appendix
\section{Proofs}
\label{app:proofs}
\setcounter{theorem}{0}
\begin{theorem}[NES and Least Squares equivalence]
    \label{thm:neslsq}
     Let $\hat{x}_{NES}$ be the Gaussian $k$-query NES estimator of a $d$-dimensional gradient
    $\bm{g}$ and let $\hat{x}_{LSQ}$ be the minimal-norm $k$-query least-squares estimator of $\bm{g}$. For any $p>0$, with probability at least $1-p$ we have that 
    \[
	\inner{\hat{x}_{LSQ}}{\bm{g}} -
	\inner{\hat{x}_{NES}}{\bm{g}} \leq
	O\left(\sqrt{\frac{k}{d} \cdot \log^3\left(\frac{k}{p}\right)}\right)\norm{g}^2,
    \]
    and in particular,
    \[
	\inner{\hat{x}_{LSQ}}{\bm{g}} -
	\inner{\hat{x}_{NES}}{\bm{g}} \leq
	8\sqrt{\frac{2k}{d}\cdot\log^3\left(\frac{2k+2}{p}\right)}\left(1 +
	\frac{\kappa}{\sqrt{d}}\right) ||g||^2
	\]
	with probability at least $1-p$, where
	\[
	\kappa \leq 2\sqrt{\log\left(\frac{2k(k+1)}{p}\right)}.
    \]
\end{theorem}

\begin{proof}
    Let us first recall our estimation setup. We have $k$ query vectors
    $\delta_i \in \mathbb{R}^d$ drawn from an i.i.d
    Gaussian distribution whose expected squared norm is one,
    i.e. $\delta_i \sim \mathcal{N}(0, \frac{1}{d}{I})$, for each $1\leq i\leq k$. Let
    the vector $\bm{y} \in \mathbb{R}^k$ denote the
    inner products of $\delta_i$s with the gradient, i.e.
    \begin{equation*}
	{y}_i := \inner{\delta_i}{\bm{g}},
    \end{equation*}
    for each $1\leq i\leq k$. We define the matrix $A$ to be a $k\times d$ matrix with
    the $\delta_i$s being its rows. That is, we have 
    \[
    A\bm{g} = \bm{y}.
    \]
    Now, recall that the closed forms of the two estimators we are interested in are given by
    \begin{align*}
	\hat{x}_{NES} &= A^T\bm{y} = A^TA\bm{g} \\
	\hat{x}_{LSQ} &=A^T(AA^T)^{-1}\bm{y} = A^T(AA^T)^{-1}A\bm{g},
	\end{align*}
	which implies that
	\begin{align*}
	\inner{\hat{x}_{NES}}{\bm{g}} &=
	\bm{g}^TA^TA\bm{g} \\
	\inner{\hat{x}_{LSQ}}{\bm{g}} &=
	\bm{g}^TA^T(AA^T)^{-1}A\bm{g}.
    \end{align*}

   \noindent We can bound the difference between these two inner products as
    \begin{align}\nonumber
	\inner{\hat{x}_{LSQ}}{\bm{g}} -
	\inner{\hat{x}_{NES}}{\bm{g}} &= 
	\bm{g}^TA^T\left[(AA^T)^{-1} - I\right]A\bm{g} \\ \nonumber
	&\leq 
	\norm{\bm{g}^TA^T}\norm{(AA^T)^{-1} - I} \norm{A\bm{g}} \\\label{eq:bound_diff}
	&\leq \norm{(AA^T)^{-1} - I} \norm{A\bm{g}}^2.
    \end{align}
    
    \noindent Now, to bound the first term in \eqref{eq:bound_diff}, observe that
     \begin{align*}
	 (AA^T)^{-1} = \left(I-(I-AA^T)\right)^{-1}  &= \sum_{l=0}^\infty (I-AA^T)^l
	 \end{align*}
	 and thus
	      \begin{align*}
	 	 I-(AA^T)^{-1} &=\sum_{l=1}^\infty (AA^T-I)^l.
	 	 \end{align*}
	 	 (Note that the first term in the above sum has been canceled out.)
	 This gives us that
	 \begin{align*}
	 \norm{I - (AA^T)^{-1}} &\leq
	 \sum_{l=1}^\infty \norm{AA^T-I}^l \\
	 &\leq \frac{\norm{AA^T-I}}{1-\norm{AA^T-I}} \\
	 &\leq 2\norm{AA^T-I},
     \end{align*}
     as long as $\norm{AA^T-I} \leq
          \frac{1}{2}$ (which, as we will see,  is indeed the case with high probability).

Our goal thus becomes bounding $\norm{AA^T-I}=\lambda_{max}(AA^T-I)$, where $\lambda_{max}(\cdot)$ denotes the largest (in absolute value) eigenvalue. Observe that $AA^T$ and $-I$ commute and are simultaneously  diagonalizable. As a result, for any $1\leq i\leq k$, we have that the $i$-th largest eigenvalue $\lambda_i(AA^T-I)$ of $AA^T-I$ can be written as 
\[
\lambda_i{(AA^T-I)} = \lambda_i{(AA^T)} + \lambda_i{(-I)}_i = \lambda_i{(AA^T)} - 1.
\]
So, we need to bound 
\[
    \lambda_{max}(AA^T-I) = \max\left\{\lambda_1(AA^T)-1,
    1-\lambda_{k}(AA^T)\right\}
\]

To this end, recall that $\mathbb{E}[AA^T] = I$ (since the rows of $A$ are sampled from the distribution $\mathcal{N}(0, \frac{1}{d}{I})$), and thus, by the covariance estimation theorem of Gittens and Tropp \cite{tropp} (see Corollary 7.2) (and union bounding over the two relevant events), we have that
    \begin{align*}
	{\Pr}(\lambda_{max}(AA^T-I)\geq \varepsilon) & =
	{\Pr}(\lambda_{1}(AA^T)\geq 1+\varepsilon \text { or }
	\lambda_{k}(AA^T)\geq 1-\varepsilon)\\
	& =  {\Pr}(\lambda_{1}(AA^T)\geq \lambda_1(I)+\varepsilon
	\text{ or } \lambda_{k}(AA^T)\geq \lambda_k(I)-\varepsilon) \leq
	2k\cdot\exp\left(-\frac{d\varepsilon^2}{32k}\right).
    \end{align*}
    Setting 
    $$\varepsilon = \sqrt{\frac{32k\log(2(k+1)/p)}{d}},$$ 
    ensuring that $\varepsilon \leq \frac{1}{2}$, gives us 
    \begin{align*}
	\Pr\left(\lambda_{max}(AA^T) - 1 \geq
	\sqrt{\frac{32k\log(2(k+1)/p)}{d}}\right) \leq \frac{k}{k+1}p.
    \end{align*}
    and thus
    \begin{equation}	
	\label{eq:comp}
	\norm{(AA^T)^{-1}-I} \leq 
	\sqrt{\frac{32k\log(2(k+1)/p)}{d}},
    \end{equation}
    with probability at least $1 - \frac{k}{k+1}p$.
    
    To bound the second term in \eqref{eq:bound_diff}, we note that all the
    vectors $\delta_i$ are chosen independently of the vector $\bm{g}$ and each
    other. So, if we consider the set
    $\{\hat{{g}},\hat{\delta_1},\ldots,\hat{\delta_k}\}$ of $k+1$
    corresponding {\em normalized} directions, we have (see, e.g.,
    \citep{gorban}) that the probability that any two of them have the
    (absolute value of) their inner product be larger than some
    $\varepsilon'= \sqrt{\frac{2\log(2(k+1)/p)}{d}}$ is at most
    \[
    \exp\left\{-(k+1)^2e^{-d(\varepsilon')^2 /
    2}\right\}=\exp\left\{-2\frac{k+1}{p}\right\} \leq \frac{p}{2(k+1)}.
    \]

    \noindent On the other hand, we note that each $\delta_i$ is a random
    vector sampled from the distribution $\mathcal{N}(0, \frac{1}{d}\bm{I}_d)$, so we have that (see, e.g., Lemma 1 in ~\citep{massart}), for any $1\leq i \leq k$ and any $\varepsilon''>0$,
    \begin{equation*}
	\Pr\left(||\delta_i||^2 \geq 1 + \varepsilon''\right) \leq \exp\left\{
	-\frac{(\varepsilon'')^2 d}{4} \right\}.
    \end{equation*}
    Setting 
    $$\varepsilon'' = 2\sqrt{\frac{\log(2k(k+1)/p)}{d}}$$ yields
    \begin{equation*}
	P\left(||\delta_i||^2 \geq 1+2\sqrt{\frac{\log(2(k+1)k/p)}{d}}\right) \leq
	\frac{p}{2k(k+1)}.
    \end{equation*}
    Applying these two bounds (and, again, union bounding over all the
    relevant events), we get that
    \begin{align*}
	\norm{A\bm{g}}^2 &= \sum_{i=1}^k (A\bm{g})_i^2  \\
			 &\leq d\cdot\left(\frac{2\log \left(\frac{2(k+1)}{p}
		\right)}{d}\right)\left(1+2\sqrt{\frac{\log\left(
	\frac{2k(k+1)}{p}\right)}{d}}\right)\norm{g}^2 \\
	&\leq 2\log \left(\frac{2(k+1)}{p}
		\right)\left(1+2\sqrt{\frac{2\log\left(
	\frac{2(k+1)}{p}\right)}{d}}\right)\norm{g}^2  \\
    \end{align*}
    with probability at most $\frac{p}{k+1}$.  \\

    Finally, by plugging the above bound and the bound \eqref{eq:comp} into the bound \eqref{eq:bound_diff}, we obtain that 
    \begin{align*}
	\inner{\hat{x}_{LSQ}}{\bm{g}} -
	\inner{\hat{x}_{NES}}{\bm{g}} &\leq
	\left(\sqrt{\frac{32k\log(2(k+1)/p)}{d}}\right)
	\cdot 2\log \left(\frac{2(k+1)}{p}
		\right)\left(1+2\sqrt{\frac{2\log\left(
	\frac{2(k+1)}{p}\right)}{d}}\right)\norm{g}^2  \\
	&\leq 
	8\sqrt{\frac{2k}{d}\cdot\log^3\left(\frac{2k+2}{p}\right)}\left(1 +
	\frac{\kappa}{\sqrt{d}}\right) ||g||^2,
    \end{align*}
with probability $1-p$, where
\[
    \kappa = 2\sqrt{\log\left(\frac{2k(k+1)}{p}\right)}.
\]
This completes the proof.
\end{proof}

\clearpage
\begin{theorem}[Least-Squares Optimality]
For a fixed projection matrix $A$ and under the following observation model of isotropic Gaussian
noise: $\bm{y} = A\bm{g} + \vec{\varepsilon}\text{ where }\bm{\varepsilon} \sim
\mathcal{N}(\bm{0},\varepsilon \bm{Id})$, 
the least-squares estimator as in Theorem~\ref{thm:neslsq},
$\hat{x}_{LSQ} = A^T(AA^T)^{-1}\bm{y}$
is a finite-sample efficient (minimum-variance unbiased) estimator of the parameter $\bm{g}$.
\end{theorem}
\begin{proof}
    Proving the theorem requires an application of the Cramer-Rao Lower Bound theorem:
    \begin{theorem}[Cramer-Rao Lower Bound]
	Given a parameter $\theta$, an observation distribution $p(x;\theta)$, and an unbiased
	estimator $\hat{\theta}$ that uses only samples from $p(x;\theta)$, then (subject to Fisher
	regularity conditions trivially satisfied by Gaussian distributions),
	$$\text{Cov}\left[\hat{\theta} - \theta\right] = \mathbb{E}\left[(\hat{\theta} -
	\theta)(\hat{\theta} - \theta)^T\right]
	\geq \left[I(\theta)\right]^{-1}
	\text{ where $I(\theta)$ is the Fisher matrix: }
	\left[I(\theta)\right]_{ij} = -\mathbb{E}\left[\frac{\partial \log p(x;\theta)}{\partial\theta_i\partial \theta_j}\right]$$
    \end{theorem}
    Now, note that the Cramer-Rao bound implies that if the variance of the estimator $\hat{\theta}$ is the
    inverse of the Fisher matrix, $\hat{\theta}$ must be the minimum-variance unbiased estimator.
    Recall the following form of the Fisher matrix:
    \begin{align}
	I(\theta) = \mathbb{E}\left[\left(\frac{\partial \log p(x;\theta)}{\partial
	\theta}\right)\left(\frac{\partial \log p(x;\theta)}{\partial \theta}\right)^T
    \right]\label{eq:fisher}
    \end{align}

    Now, suppose we had the following equality, which we can then simplify using the preceding
    equation:
    \begin{align}
	I(\theta)\left(\hat{\theta}-\theta\right) &= \frac{\partial \log p(x;\theta)}{\partial \theta}
	\label{eq:condition} \\
	\left(I(\theta)\left(\hat{\theta}-\theta\right)\right)
	\left(I(\theta)\left(\hat{\theta}-\theta\right)\right)^T
	&= \left(\frac{\partial \log p(x;\theta)}{\partial \theta}\right)\left(\frac{\partial
	\log p(x;\theta)}{\partial \theta}\right)^T \\
	\mathbb{E}\left[\left(I(\theta)\left(\hat{\theta}-\theta\right)\right)
	\left(I(\theta)\left(\hat{\theta}-\theta\right)\right)^T\right]
	&= \mathbb{E}\left[\left(\frac{\partial \log p(x;\theta)}{\partial \theta}\right)\left(\frac{\partial
	\log p(x;\theta)}{\partial \theta}\right)^T\right] \\
	I(\theta)\mathbb{E}\left[(\hat{\theta}-\theta)(\hat{\theta}-\theta)^T\right]I(\theta) &=
	I(\theta)
    \end{align}
    Multiplying the preceding by $\left[I(\theta)\right]^{-1}$ on both the left and right sides
    yields:
    \begin{align}
	\mathbb{E}\left[(\hat{\theta}-\theta)(\hat{\theta}-\theta)^T\right] =
	\left[I(\theta)\right]^{-1},
    \end{align}
    which tells us that~\eqref{eq:condition} is a sufficient condition for finite-sample efficiency
    (minimal variance). We show that this condition is satisfied in our case, where we have
    $y \sim A\bm{g} + \varepsilon$, $\hat{\theta} = \hat{x}_{LSQ}$, and $\theta = \bm{g}$. We begin
    by computing the Fisher matrix directly, starting from the distribution of the samples $y$:
    \begin{align}
	p(y;\bm{g}) &= \frac{1}{\sqrt{(2\pi\varepsilon)^d}}
	    \exp\left\{\frac{1}{2\varepsilon}(y-A\bm{g})^T(y-A\bm{g})\right\}  \\
	\log p(y;\bm{g}) &= \frac{d}{2}\log\left(2\pi\varepsilon\right) +
	    \frac{1}{2\varepsilon}(y-A\bm{g})^T(y-A\bm{g}) \\
	\frac{\partial \log p(y;\bm{g})}{\partial \bm{g}} &= \frac{1}{2\varepsilon}
	    \left(2A^T(y-A\bm{g})\right) \\
	&= \frac{1}{\varepsilon}A^T(y-A\bm{g}) \\
    \end{align}
    Using~\eqref{eq:fisher},
    \begin{align}
	    I(\bm{g}) &= \mathbb{E}\left[\left(\frac{1}{\varepsilon}A^T(y-A\bm{g})\right)
	\left(\frac{1}{\varepsilon}A^T(y-A\bm{g})\right)^T\right] \\
	&= \frac{1}{\varepsilon^2}A^T\mathbb{E}\left[(y-A\bm{g})(y-A\bm{g})^T\right]A \\
	&= \frac{1}{\varepsilon^2}A^T(\varepsilon \bm{Id})A \\
	&= \frac{1}{\varepsilon}A^TA
    \end{align}

    Finally, note that we can write:
    \begin{align}
	I(\bm{g})(\hat{x}_{LSQ} - \bm{g}) &= \frac{1}{\varepsilon}A^TA(A^T(AA^T)^{-1}y - \bm{g}) \\
	&= \frac{1}{\varepsilon}(A^Ty - A^TA\bm{g}) \\
	&= \frac{\partial \log p(y;\bm{g})}{\partial \bm{g}},
    \end{align}
    which concludes the proof, as we have shown that $\hat{x}_{LSQ}$ satisfies the
    condition~\eqref{eq:condition}, which in turn implies finite-sample efficiency.
\end{proof}

\begin{claim}
    Applying the precise bound that we can derive from Theorem~\ref{thm:neslsq} on an
    ImageNet-sized dataset $(d = 300000)$ and using $k = 100$ queries (what
    we use in our $\ell_\infty$ threat model and ten
    times that used for our $\ell_2$ threat model),
    \[
	\inner{\hat{x}_{LSQ}}{\bm{g}} -
	\inner{\hat{x}_{NES}}{\bm{g}} \leq \frac{5}{4}||g||^2.
    \]
    For 10 queries, 
    \[
	\inner{\hat{x}_{LSQ}}{\bm{g}} -
	\inner{\hat{x}_{NES}}{\bm{g}} \leq \frac{1}{2}||g||^2.
    \]
\end{claim}

\clearpage
\section{Omitted Figures}
\label{app:omitted}

\subsection{Compressive Sensing}
Compressed sensing approaches can, in some cases, solve the optimization problem presented in Section~\ref{sec:gradient_estimation_problem}. However, these approaches require sparsity to improve over the least squares method. Here we show the lack of sparsity in gradients through a classifier on a set of canonical bases for images. In Figure~\ref{fig:sparsities}, we plot
the fraction of $\ell_2$ weight accounted for by the largest $k$ components in randomly chosen image gradients when using two canonical bases: standard and wavelet (db4). While lack of sparsity in these bases does not strictly preclude the existence of a
basis on which gradients are sparse, it suggests the lack of a fundamental
structural sparsity in gradients through a convolutional neural network. 

\begin{figure}[h!]
  \begin{center}
  
  \includegraphics{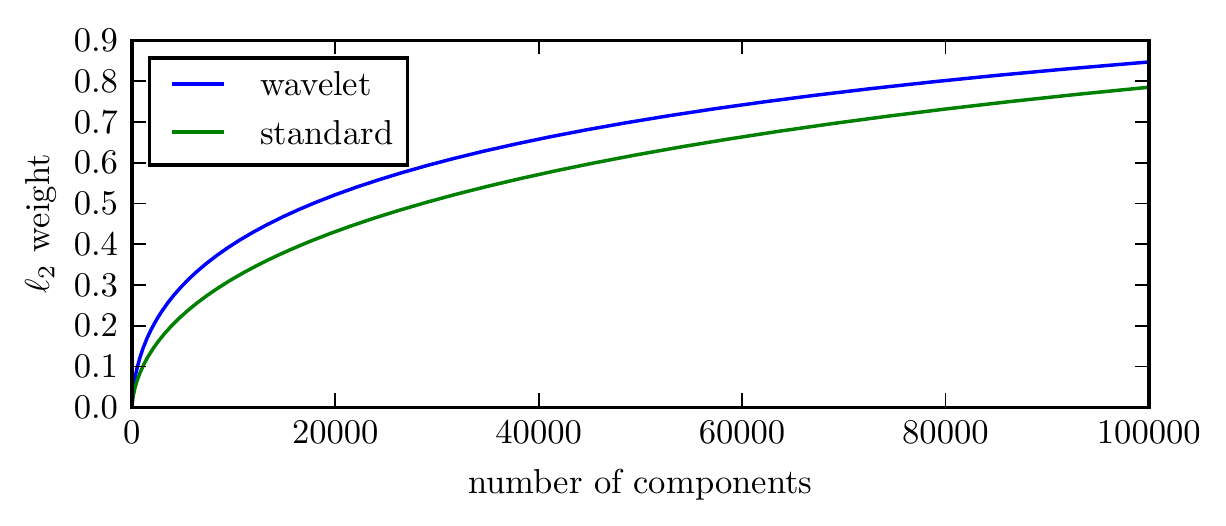}
  \end{center}
  \caption{Sparsity in standard, wavelet (db4 wavelets), and PCA-constructed bases for the gradients of 5,000 randomly chosen example images in the ImageNet validation set. The y-axis shows the mean fraction of $\ell_2$ weight held by the largest $k$ vectors over the set of 5,000 chosen images. The x-axis varies $k$. The gradients are taken through a standardly trained Inception v3 network. None of the bases explored induce significant sparsity.}
  \label{fig:sparsities}
\end{figure}

\clearpage
\subsection{Tiling}
An example of the tiling procedure applied to a gradient can be seen in Figure~\ref{fig:avg-blur}.
\begin{figure}[h!]
  \centering
  \begin{subfigure}{.45\textwidth}
    \centering
    {
      \setlength{\fboxsep}{0pt}
      \setlength{\fboxrule}{1pt}

      \fbox{\includegraphics[width=.85\linewidth]{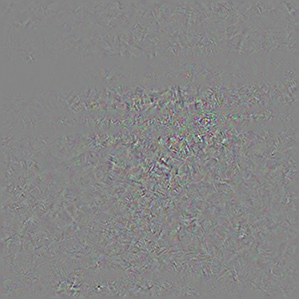}}
    }

    \caption{Gradient}
    \label{fig:avg-blur-sub1}
  \end{subfigure}
  \begin{subfigure}{.45\textwidth}
    \centering
    {
      \setlength{\fboxsep}{0pt}
      \setlength{\fboxrule}{1pt}
      \fbox{\includegraphics[width=.85\linewidth]{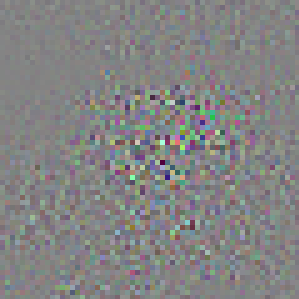}}
    }
    \caption{Tiled gradient}
    \label{fig:avg-blu-sub2}
\end{subfigure}
  \caption{Average blurred gradient with kernel size or ``tile length'' 5. The original gradient can be seen in~\ref{fig:avg-blur-sub1}, and the ``tiled'' or average blurred gradient can be seen in~\ref{fig:avg-blu-sub2}}
  \label{fig:avg-blur}
\end{figure}

\subsection{Time-dependent Priors at Higher Step Sizes}
We show in Figure~\ref{fig:x_step_priors_2} that the correlation between
successive gradients on the NES trajectory are signficantly correlated,
even at much higher step sizes (up to $\ell_2$ norm of 4.0, which is a
typical value for $\varepsilon$, the total adversarial perturbation bound
and thus an absolute bound on step size). This serves as further motivation
for the time-dependent prior.

\begin{figure}[h!]
    \begin{center}
	\includegraphics{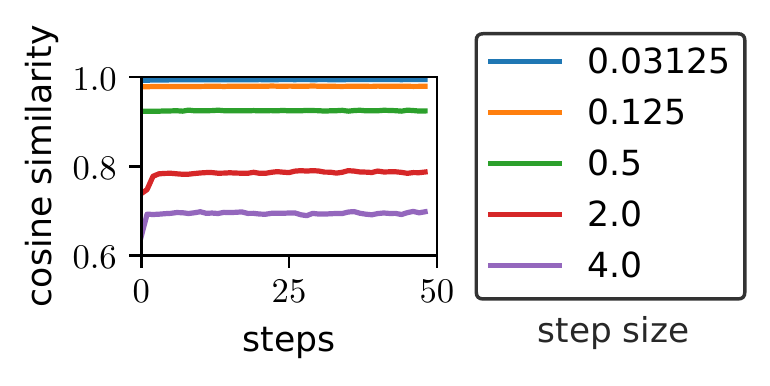}
	\caption{Figure~\ref{fig:successive_correlation} repeated for
	    several step sizes, showing that the successive correlation
	between gradients continues even at higher step sizes.}
	\label{fig:x_step_priors_2}
    \end{center}
\end{figure}

\clearpage
\section{Hyperparameters}
\label{app:hyperparameters}
\begin{table}[h]
\centering
\caption{Hyperparameters for the NES approach.}
\label{app-tab:nes-hyperparameters}
  \begin{tabular}{@{}cccc@{}}
    \toprule
    \multirow{2}{*}{\textbf{Hyperparameter}} & \phantom{x} & \multicolumn{2}{c}{\textbf{Value}} \\
    \cmidrule{3-4} && ImageNet $\ell_\infty$ & ImageNet $\ell_2$ \\
    \midrule
    Samples per step && 100 & 10 \\
    Learning Rate && 0.01 & 0.3 \\
    \bottomrule
  \end{tabular}
\end{table}

\begin{table}[h]
\centering
\caption{Hyperparameters for the bandits approach
(variables names as used in pseudocode).}
\label{app-tab:bandits-hyperparameters}
  \begin{tabular}{@{}cccc@{}}
    \toprule
    
    \multirow{2}{*}{\textbf{Hyperparameter}} & \phantom{x} & \multicolumn{2}{c}{\textbf{Value}} \\
    \cmidrule{3-4} && ImageNet $\ell_\infty$ & ImageNet $\ell_2$\\
    \midrule
    $\eta$ (OCO learning rate) && $100$ & $0.1$\\
    $h$ (Image $\ell_p$ learning rate) && $0.01$ & $0.5$ \\
    $\delta$ (Bandit exploration) && $1.0$ & $0.01$ \\
    $\eta$ (Finite difference probe) && $0.1$ & $0.01$ \\
    Tile size (Data-dependent prior only) && $(6px)^2$ & $(6px)^2$ \\
    \bottomrule
  \end{tabular}
\end{table}

\clearpage
\section{Full Results}
\label{app:results}

\begin{figure}[h]
  \include{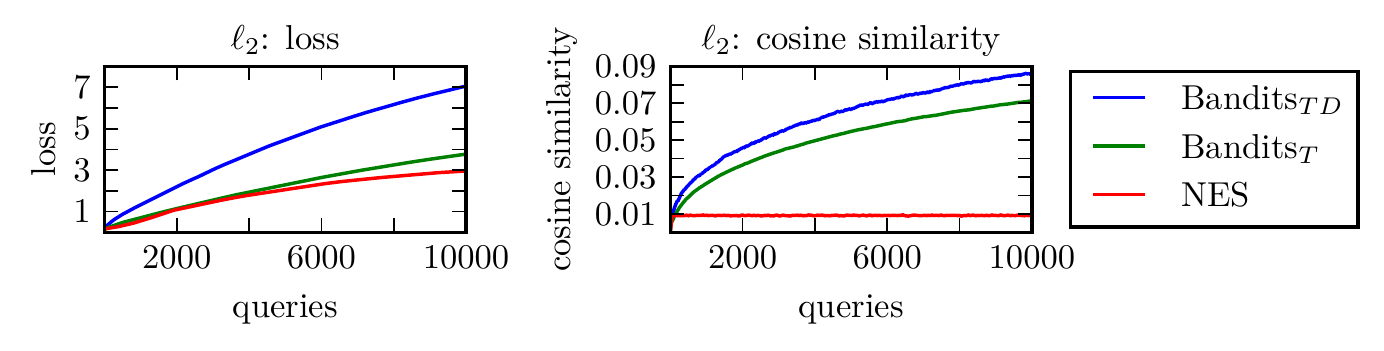}
  \vspace{-6em}
  \include{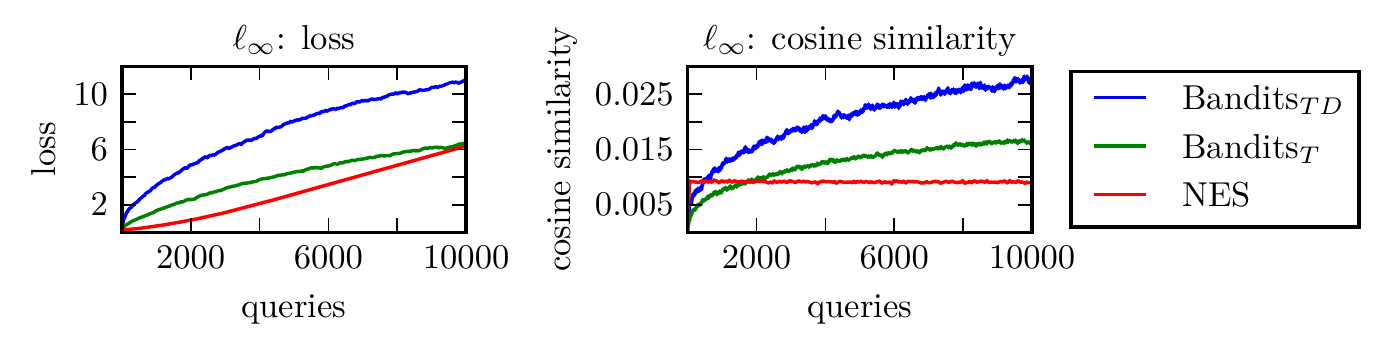}
  \caption{\textit{Average loss and cosine distance versus number of
      queries used over the approaches' optimization trajectories in the
      two threat models (averaged over 100 images).}}
  \label{fig:multistep-plots}
\end{figure}

\begin{figure}[h]
  \include{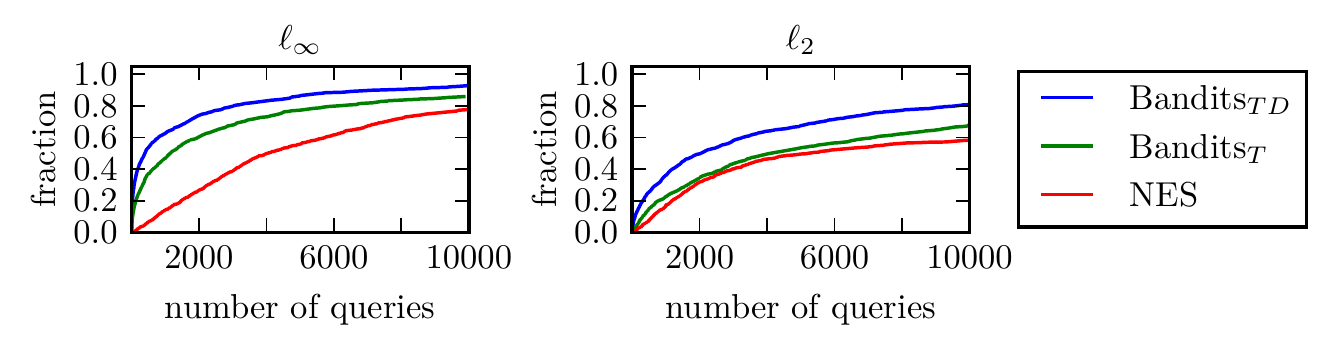}
  \caption{\textit{Cumulative distribution functions for the number of queries required
      to create an adversarial example in the $\ell_2$ and $\ell_\infty$
      settings for the NES, bandits with time prior (Bandits$_{T}$), and
      bandits with time and data-dependent priors (Bandits$_{TD}$)
      approaches. Note that the CDFs do not converge to one, as the
      approaches sometimes cannot find an adversarial example in less than
      10,000 queries.}}
  \label{fig:cdf}
\end{figure}

\begin{figure}[h]
  \begin{tikzpicture}[every plot/.append style={very thick}]

\begin{groupplot}[group style={group size=2 by 1, horizontal sep=2cm}]
\nextgroupplot[
title={Avg. queries by success \% ($\ell_\infty$)},
ylabel={number of queries},
x label style={at={(axis description cs:0.5,0.05)},anchor=north},
y label style={at={(axis description cs:0.05,0.5)},anchor=north},
ytick={0,500,1000,1500,2000},
yticklabels={0,,,,2k},
scaled y ticks=false,
axis line style=thick,
xlabel={success rate},
ymin=0, ymax=2000,
xmin=0, xmax=1.0,
axis on top,
xtick={0,0.5,1.0},
xticklabels={0,50,100},
tick pos=both,
width=0.35\textwidth,
height=3.2cm
]
\addplot [dotted, blue] table[x index=1, y index=0]{fig/data/avgq-vs-success/nes-linf-cdf.csv};
\addplot [dashed, green!50.0!black] table[x index=1, y index=0] {fig/data/avgq-vs-success/banditst-linf-cdf.csv};
\addplot [red] table[x index=1, y index=0] {fig/data/avgq-vs-success/banditstd-linf-cdf.csv};
\nextgroupplot[
title={Avg. queries by success \% ($\ell_2$)},
ylabel={avg. queries},
x label style={at={(axis description cs:0.5,0.05)},anchor=north},
y label style={at={(axis description cs:0.05,0.5)},anchor=north},
ytick={0,500,1000,1500,2000,2500,3000},
yticklabels={0,,,,,,3k},
axis line style=thick,
scaled y ticks=false,
xlabel={success rate},
ymin=0, ymax=3000,
xmin=0, xmax=1.0,
axis on top,
xtick={0,0.5,1.0},
xticklabels={0,50,100},
tick pos=both,
legend entries={{NES},{Bandits$_T$ (time prior)},{Bandits$_TD$ (time $+$ data)}},
legend cell align={left},
legend style={at={(1.2, 0.5)}, anchor=west}, 
width=0.35\textwidth,
height=3.2cm
]
\addlegendimage{no markers, blue, dotted}
\addlegendimage{no markers, green!50.0!black, dashed}
\addlegendimage{no markers, red}
\addplot [dotted, mark options={scale=0.1}, blue] table[x index=1, y index=0]{fig/data/avgq-vs-success/nes-l2-cdf.csv};
\addplot [dashed, green!50.0!black] table[x index=1, y index=0]{fig/data/avgq-vs-success/banditst-l2-cdf.csv};
\addplot [red] table[x index=1, y index=0] {fig/data/avgq-vs-success/banditstd-l2-cdf.csv};
\end{groupplot}

\end{tikzpicture}
\vspace*{-2.5em}
  \caption{\textit{The average number of queries used per successful image for each
      method when reaching a specified success rate: we compare NES~\cite{query-efficient},
  Bandits$_T$ (our method with time prior only), and Bandits$_{TD}$ (our
  method with both data and time priors)  and find that our methods strictly
  dominate NES---that is, for any desired sucess rate, our
  methods take strictly less queries per successful image than NES.}}
  \label{fig:query-cdf}
\end{figure}
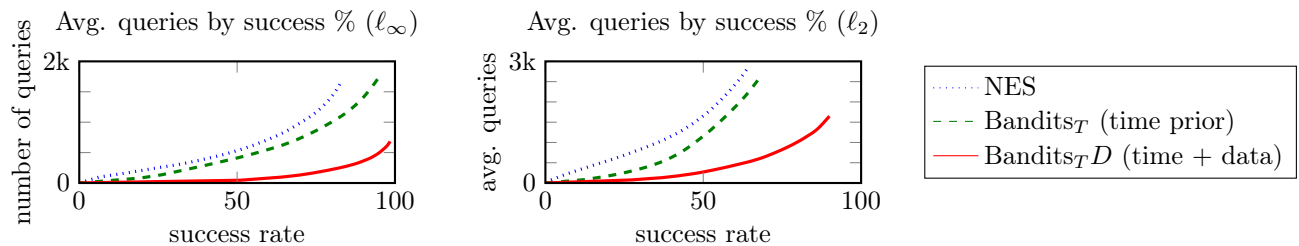

\clearpage
\section{Results for other Classifiers}
\label{app:classifiers}

Here, we give results for the ImageNet dataset, comparing our best method
(Bandits$_{TD}$) and NES for Inception-v3 (also shown in
Table~\ref{tab:multistep}), VGG16, and ResNet50 classifiers. Note that we
do not fine-tune the hyperparameters to the new classifiers, but simply use
the hyperparameters found for Inception-v3. Nevertheless, our best method
consistently outperforms NES on black-box attacks.

\begin{table}[h]
\centering
\caption{Summary of effectiveness of $\ell_\infty$ and $\ell_2$ ImageNet
attacks on Inception v3, ResNet-50, and VGG16 (I, R, V) using NES and
bandits with time and data-dependent priors (Bandits$_{TD}$). Note that
in the first column, the average number of queries  is calculated only over
successful attacks, and we enforce a query limit of 10,000 queries. For
purposes of direct comparison, the last column calculates the average
number of queries used for only the images that NES (previous SOTA) was
successful on. Our most powerful attack uses 2-4 times fewer queries, and
fails 2-5 times less often.
}
{\small
  \begin{tabular}{@{}ccccccccccccccc@{}}
    \toprule
    \multirow{2}{*}{} & \multirow{2}{*}{\textbf{Attack}} & 
    \multicolumn{3}{c}{\textbf{Avg. Queries}} & \phantom{x} &
    \multicolumn{3}{c}{\textbf{Failure Rate}} & \phantom{x} &
    \multicolumn{3}{c}{\textbf{\#Q on NES Success}} \\ 
    \cmidrule{3-5} \cmidrule{7-9} \cmidrule{11-13} && I & R & V && I & R & V && I & R & V\\
    \midrule
    \multirow{2}{*}{$\ell_2$} & NES &
     2938 & 2193 & 1244 &&
     34.4\% & 10.1\% & \textbf{11.6\%} &&
     2938 & 2193 & 1244 \\
    & \textbf{Bandits$_{\mathbf{TD}}$} &
    \textbf{1858} & \textbf{993} & \textbf{594} &
    &\textbf{15.5\%} &  \textbf{9.7\%} & 17.2\% &&
     \textbf{999} & \textbf{1195} & \textbf{1219} \\
    \midrule 
    \multirow{2}{*}{$\ell_\infty$} & NES &
      1735 & 1397  & 764  &&
      22.2\% & 10.4\% & 10.5\% &&
      1735 & 1397 & 764 \\
    & \textbf{Bandits$_{\mathbf{TD}}$} &
    \textbf{1117} & \textbf{722} & \textbf{370} &
    &\textbf{4.6\%} &  \textbf{3.4\%} & \textbf{8.4\%} &&
     \textbf{703} & \textbf{594} & \textbf{339} \\
    \bottomrule
  \end{tabular}
  }
\end{table}

\end{document}